\documentclass{article}
\usepackage{spconf,amsmath,graphicx,xcolor,amssymb,amsthm,enumerate,ifthen,overpic,bm,aliascnt}
\usepackage[linesnumbered,vlined,boxed,commentsnumbered]{algorithm2e}
\ninept

\def\nset{{\mathbb{N}}}
\def\rset{\mathbb R}

\def\Zset{\mathsf{Z}}
\def\Zsigma{\mathcal{Z}}

\def\rmd{\mathrm{d}}
\def\argmin{\operatorname{Argmin}}

\def\max{\mathrm{max}}

\def\1{\mathbbm{1}}

\def\PP{\mathbb{P}} 
\def\PE{\mathbb{E}} 
\def\bPE{\overline{\mathbb{E}}} 

\newcommand{\F}{\mathcal{F}} 

\newcommand{\pscal}[2]{\left\langle#1,#2\right\rangle}

\newcommand{\eqdef}{\ensuremath{\stackrel{\mathrm{def}}{=}}}

\newcommand{\kouter}{k_\mathrm{out}}
\newcommand{\kin}{k_\mathrm{in}}

\newcommand\init{\mathrm{init}}
\newcommand{\R}{\mathsf{R}}
\newcommand{\loss}[1]{\ensuremath{\mathcal{L}_{#1}}}
\newcommand{\bars}{\bar{s}}
\newcommand{\s}{s}
\newcommand{\hatS}{\widehat{S}}
\newcommand{\Smem}{\mathsf{S}}

\newcommand{\param}{\theta}
\newcommand{\Param}{\Theta}
\newcommand{\map}{\mathsf{T}}
\newcommand{\lyap}{\operatorname{W}}
\newcommand{\batch}{\mathcal{B}}
\newcommand{\lbatch}{\mathsf{b}}
\newcommand{\pas}{\gamma}

\newcommand\sequence[3] {\ifthenelse{\equal{#3}{}}{\ensuremath{\{
#1_{#2}\}}}{\ensuremath{\{ #1^{#2}, \eqsp #2 \in #3 \}}}}
\newcommand\sequencedown[3] {\ifthenelse{\equal{#3}{}}{\ensuremath{\{
#1_{#2}\}}}{\ensuremath{\{ #1_{#2}, \eqsp #2 \in #3 \}}}}

\def\eqsp{\;}
\newcommand{\ie}{i.e.}

\newcommand{\ooint}[1]{\left(#1\right)}
\newcommand{\ccint}[1]{\left[#1\right]}

\newtheorem{assumption}{A\hspace{-3pt}}
\newtheorem{theorem}{Theorem}
\newaliascnt{proposition}{theorem}
\newtheorem{proposition}[proposition]{Proposition}
\aliascntresetthe{proposition}
\newaliascnt{lemma}{theorem}
\newtheorem{lemma}[lemma]{Lemma}
\aliascntresetthe{lemma}
\newaliascnt{corollary}{theorem}
\newtheorem{corollary}[corollary]{Corollary}
\aliascntresetthe{corollary}

\newaliascnt{definition}{theorem}

\aliascntresetthe{definition}
\newaliascnt{remark}{theorem}

\aliascntresetthe{remark}

\title{GEOM-SPIDER-EM: faster variance reduced stochastic Expectation
  Maximization for nonconvex finite-sum optimization}

\name{Gersende Fort$^{\diamond}$\thanks{Part of this work is funded by the Fondation Simone and Cino Del Duca under the program OpSiMorE} \qquad Eric Moulines$^{\star}$ \qquad Hoi-To Wai$^{\dagger}$}
\address{$^{\diamond}$ Institut Math\'ematique de Toulouse, Universit\'e de Toulouse; CNRS
  UPS, F-31062 Toulouse Cedex, France \\
  $^{\star}$  Centre de Math\'ematiques Appliqu\'ees;
Ecole Polytechnique;
91128 Palaiseau Cedex, France\\
$^{\dagger}$ Department of SEEM;
The Chinese University of Hong Kong;
Shatin, Hong Kong}

\begin{document}

\maketitle

\begin{abstract}
  The Expectation Maximization (EM) algorithm is a key reference for
  inference in latent variable models; unfortunately, its
  computational cost is prohibitive in the large scale learning
  setting. In this paper, we propose an extension of the Stochastic
  Path-Integrated Differential EstimatoR EM (SPIDER-EM) and derive
  complexity bounds for this novel algorithm, designed to solve smooth
  nonconvex finite-sum optimization problems. We show that it reaches
  the same state of the art complexity bounds as SPIDER-EM; and
  provide conditions for a linear rate of convergence. Numerical
  results support our findings.
\end{abstract}

\begin{keywords}
Large scale learning, Latent variable analysis, Expectation Maximization,
Stochastic nonconvex optimization, Variance reduction.
\end{keywords}
\section{Introduction}
\label{sec:intro}
Intelligent processing of large data set and efficient learning of high-dimensional
models require new optimization
algorithms designed to be robust to big data and complex models
era (see
e.g.~\cite{slavakis:giannakis:mateos:2014,buhlmann2016handbook,hardle2018handbook}). This
paper is concerned with stochastic optimization of a nonconvex
finite-sum smooth objective function
\begin{equation}\label{eq:problem}
\argmin_{\param\in \Param}  F(\param), \qquad F(\param) \eqdef \frac{1}{n} \sum_{i=1}^n \loss{i}(\param) +  \R(\param) \eqsp,
\end{equation}
when $\Param \subseteq \rset^d$ and $F$ cannot be explicitly evaluated
(nor its gradient).  Many statistical learning problems can be cast
into this framework, where $n$ is the number of observations or
examples, $\loss{i}$ is a loss function associated to example $\#i$
(most of often a negated log-likelihood), and $\R$ is a penalty term
promoting sparsity, regularity, etc.. Intractability of $F(\param)$
might come from two sources. The first, referred to as {\em large
  scale learning} setting, is that the number $n$ is very large so
that the computations involving a sum over $n$ terms should be either
simply avoided or sparingly used during the run of the optimization
algorithm (see e.g.~\cite{slavakis:mateos:kim:giannakis:2014} for an
introduction to the bridge between large scale learning and stochastic
approximation;
see~\cite{lecun:bernhard:etal:1990,lecun:bottou:bengio:haffner:1998}
for applications to training of deep neural networks for signal and
image processing; and more generally, empirical risk minimization in
machine learning is a matter for \eqref{eq:problem}). The second is
due to the presence of latent variables: for any $i$, the function
$\loss{i}$ as a (high-dimensional) integral over latent variables.
Such a latent variable context is a classical statistical modeling:
for example as a tool for solving inference in mixture
models~\cite{mclachlan:2000}, for the definition of mixed models
capturing variability among examples~\cite{jiang:2007} or for modeling
hidden and/or missing variables (see e.g. applications in text
modeling through latent Dirichlet
allocation~\cite{blei:ng:jordan:2003}, in audio source
separation~\cite{kounades:etal:2016,weisberg:etal:2019}, in
hyper-spectral imaging~\cite{lin:etal:2016}).

In this contribution, we address the two levels of intractability in
the case $\loss{i}$ is of the form
\begin{equation}\label{eq:def:loss}
  \loss{i}(\param) \eqdef - \log \int_\Zset h_i(z) \exp\left(
    \pscal{\s_i(z)}{\phi(\param)} - \psi(\param)\right) \mu(\rmd z) \eqsp.
\end{equation}
This setting in particular covers the case when
$\sum_{i=1}^n \loss{i}(\param)$ is the negated log-likelihood of the
observations $(Y_1, \cdots, Y_n)$, the pairs observation/latent
variable $\{(Y_i,Z_i), i \leq n\}$ are independent, and the
distribution of the complete data $(Y_i, Z_i)$ given by
$(y_i,z) \mapsto h_i(z) \exp\left( \pscal{\s_i(z)}{\phi(\param)} -
  \psi(\param)\right) \mu(\rmd z)$
is from the curved exponential family. Gaussian mixture models are
typical examples, as well as mixtures of distributions from the curved
exponential family.

In the framework \eqref{eq:problem}-\eqref{eq:def:loss}, a
Majorize-Minimization approach through the Expectation-Maximization
(EM) algorithm~\cite{dempster:1977} is standard; unfortunately, the
computational cost of the batch EM can be prohibitive in the large
scale learning setting. Different strategies were proposed to address
this
issue~\cite{Neal:hinton:1998,cappe:moulines:2009,chen:etal:2018,karimi:etal:2019,fort:moulines:wai:2020}:
they combine mini-batches processing, Stochastic Approximation (SA)
techniques~(see
e.g.~\cite{benveniste:priouret:metivier:1990,borkar:2008}) and
variance reduction methods.

The first contribution of this paper is to provide a novel algorithm,
the generalized Stochastic Path-Integrated Differential EstimatoR EM
({\tt g-SPIDER-EM}), which is among the variance reduced stochastic EM
methods for nonconvex finite-sum optimization of the form
\eqref{eq:problem}-\eqref{eq:def:loss}; the generalizations allow a
reduced computational cost without altering the convergence
properties. The second contribution is the proof of complexity bounds,
that is the number of parameter updates (M-step) and the number of conditional
expectations evaluations (E-step), in order to reach
$\epsilon$-approximate stationary points; these bounds are derived for
a specific form of {\tt g-SPIDER-EM}: we show that its complexity bounds are the
same as those of {\tt SPIDER-EM}, bounds which are state of the art
ones and overpass all the previous ones. Linear convergence rates are
proved under a Polyak-{\L}ojasiewicz condition. Finally,
numerical results support our findings and provide insights on how to
implement {\tt g-SPIDER-EM} in order to inherit the properties of {\tt
  SPIDER-EM} while reducing the computational cost.

{\bf Notations} For two vectors $a,b \in \rset^q$, $\pscal{a}{b}$ is
the scalar product, and $\|\cdot \|$ the associated norm. For a matrix
$A$, $A^T$ is its transpose. For a positive integer $n$, set
$[n]^\star \eqdef \{1, \cdots, n\}$ and $[n] \eqdef \{0, \cdots,
n\}$. $\nabla f$ denotes the gradient of a differentiable function
$f$. The minimum of $a$ and $b$ is denoted by $a \wedge b$. Finally,
we use standard big $O$ notation to leave out constants.

\section{EM-based methods in the expectation space}
We begin by formulating the model assumptions:
\begin{assumption} \label{hyp:model} $\Param \subseteq \rset^d$ is a
   convex set.  $(\Zset, \Zsigma)$ is a measurable space and $\mu$ is
   a $\sigma$-finite positive measure on $\Zsigma$. The functions $\R:
   \Param \to \rset$, $\phi : \Param \to \rset^q$, $\psi: \Param \to
   \rset$, $\s_i: \Zset \to \rset^q$, $h_i: \Zset \to \rset_+$ for all
   $i \in [n]^\star$ are measurable. For any $\param \in \Param$ and
   $i \in [n]^\star$, $|\loss{i}(\param)| < \infty$.
\end{assumption}
For any $\param \in \Param$ and $i \in [n]^\star$, define the
posterior density of the latent variable $Z_i$ given the observation
$Y_i$:
\begin{equation}
\label{eq:def:posterior}
p_i(z; \param) \eqdef h_i(z) \exp\left( \pscal{\s_i(z)}{\phi(\param) }
- \psi(\param) + \loss{i}(\param) \right) \eqsp,
\end{equation}
note that the dependence upon $y_i$ follows through the index $i$ in the above.
Set
\begin{equation} \label{eq:bars} \bars_i(\param) \eqdef \int_\Zset
  \s_i(z) \ p_i(z;\param) \mu(\rmd z), \quad \bars(\param) \eqdef
  n^{-1} \sum_{i=1}^n \bars_i(\param) \eqsp.
\end{equation}
\begin{assumption} \label{hyp:bars} The expectations $\bars_i(\param)$
  are well defined for all $\param \in \Param$ and $i \in [n]^\star$.
  For any $s \in \rset^q$, $ \argmin_{\param \in \Param} \ \left(
  \psi(\param) - \pscal{s}{\phi(\param)} + \R(\param) \right)$ is a
  (non empty) singleton denoted by $\{\map(s)\}$.
\end{assumption}
 EM is an iterative algorithm: given a current value $\tau_k \in
 \Param$, the next value is $\tau_{k+1} \leftarrow \map \circ
 \bars(\tau_k)$. It combines an expectation step which boils down to
 the computation of $\bars(\tau_k)$, the conditional expectation of
 $\s(z)$ under $p(\cdot; \tau_k)$; and a maximization step which
 corresponds to the computation of the map $\map$.  Equivalently, by
 using $\map$ which maps $\rset^q$ to $\Param$, it can be described in
 the {\em expectation space}
 (see~\cite{Delyon:lavielle:moulines:1999}): given the current value
 $\bars^k \in \bars(\Param)$, the next value is $\bars^{k+1}
 \leftarrow \bars \circ \map(\bars^k)$. 
 
 In this paper, we see EM as an
 iterative algorithm operating in the expectation space.  In that case, the
 fixed points of the EM operator $\bars \circ \map$ are the roots of
 the function $h$
\begin{equation} \label{eq:field:h} h(s) \eqdef \bars \circ \map(s) -
  s \eqsp.
\end{equation}
 EM possesses a Lyapunov function: in the parameter space, it is the
 objective function $F$ where by definition of the EM sequence, it
 holds $F(\tau_{k+1}) \leq F(\tau_k)$; in the expectation space, it is
 $\lyap \eqdef F \circ \map$, and $\lyap(\bars^{k+1}) \leq
 \lyap(\bars^k)$ holds.  In order to derive complexity bounds, regularity
 assumptions are required on $\lyap$: 
\begin{assumption} \label{hyp:regV} The functions $\phi$, $\psi$ and
  $\R$ are continuously differentiable on $\Param^v$, where $\Param^v$
  is a neighborhood of $\Param$.  $\map$ is continuously
  differentiable on $\rset^q$. The function $F$ is continuously
  differentiable on $\Param^v$ and for any $\param \in \Param$,
  $\nabla F(\param) = - \nabla \phi(\param)^T \bars(\param) + \nabla
  \psi(\param) + \nabla \R(\param)$. For any $s \in \rset^q$, $B(s)
  \eqdef \nabla{\left(\phi \circ \map \right)}(s)$ is a symmetric $q
  \times q$ matrix and there exist $0< v_{min} \leq v_{max}< \infty $
  such that for all $s\in \rset^q$, the spectrum of $B(s)$ is in
  $\ccint{v_{min}, v_{max}}$.  For any $i \in [n]^\star$, $\bars_i
  \circ \map$ is globally Lipschitz on $\rset^q$ with constant $L_i$.
  The function $ s \mapsto \nabla (F \circ \map)(s) = B(s) \left(
  \bars \circ T(s) -s \right)$ is globally Lipschitz on $\rset^q$ with
  constant $L_{\dot \lyap}$.
\end{assumption}
\noindent A\ref{hyp:regV} implies that
 $\lyap$ has globally Lipschitz
 gradient and $\nabla W(s) = -B(s) h(s)$ for some positive definite
 matrix $B(s)$ (see e.g.~\cite[Lemma
   2]{Delyon:lavielle:moulines:1999}; see also \cite[Propositions~1
   and~2]{fort:gach:moulines:2020}). Note that this implies that
 $\nabla W(s^\star) =0$ iff $h(s^\star)=0$.

Unfortunately, in the large scale learning setting (when $n \gg 1$), 
EM can not be easily applied since each iteration involves $n$
conditional expectations (CE) evaluations through $\bars=n^{-1}
\sum_{i=1}^n \bars_i$.  {\em Incremental} EM techniques have been
proposed to address this issue: the most straightforward approach
amounts to use a SA scheme with mean field $h$ since. Upon noting that
$h(s) = \PE\left[\bars_I \circ \map(s)\right]-s$ where $I$ is a
uniform random variable (r.v.) on $[n]^\star$, the fixed points of the
EM operator $\bars \circ \map$ are those of the SA scheme
\begin{equation}
\label{eq:SA:scheme}
\hatS_{k+1} = \hatS_k + \pas_{k+1} \bigg( \lbatch^{-1} \sum_{i \in
    \batch_{k+1}} \bars_{i} \circ \map(\hatS_k) - \hatS_k \bigg)
\end{equation}
where $\{\pas_k, k \geq 0 \}$ is a deterministic positive step size
sequence, and $\batch_{k+1}$ is sampled in $[n]^\star$ independently
from the past of the algorithm. This forms the basis of {\tt
  Online-EM} proposed by \cite{cappe:moulines:2009} (see also
\cite{liang2009online}).  Variance reduced versions were also proposed
and studied: Incremental EM ({\tt
  i-EM})~\cite{Neal:hinton:1998,gunawardana:2005}, Stochastic EM with
variance reduction ({\tt sEM-vr})~\cite{chen:etal:2018}, Fast
Incremental EM \cite{karimi:etal:2019,fort:gach:moulines:2020} ({\tt
  FIEM}) and more recently, Stochastic Path-Integrated Differential
EstimatoR EM ({\tt SPIDER-EM}) \cite{fort:moulines:wai:2020}. 

As shown
in \cite[section~2.3]{fort:gach:moulines:2020}, these algorithms can
be seen as a combination of SA with \emph{control variate}: upon
noting that $h(s) = h(s) + \PE[U]$ for any r.v. $U$ such that $\PE[U]
=0$, {\em control variates within SA} procedures replace
\eqref{eq:SA:scheme} with
\[
\hatS_{k+1} = \hatS_k + \pas_{k+1} \bigg( \lbatch^{-1} \sum_{i \in
    \batch_{k+1}} \bars_{i} \circ \map(\hatS_k) + U_{k+1} - \hatS_k
\bigg)
\]
for a choice of $U_{k+1}$ such that the new algorithm has better
properties (for example, in terms of complexity - see the end of
Section~\ref{sec:gspiderem}).

Lastly, we remark that A\ref{hyp:model}--A\ref{hyp:regV} are common assumptions satisfied by
many statistical models such as the Gaussian Mixture Model; see \cite{fort:moulines:wai:2020} for a rigorous justification of these assumptions. 

\section{The Geom-SPIDER-EM algorithm}
\label{sec:gspiderem}
\vspace{-0.3cm}
\begin{algorithm}[htbp]
  \KwData{ $\kouter \in \nset^\star$; $\hatS_\init \in \rset^q$;
    $\xi_t \in \nset^\star$ for $t \in [\kouter]^\star$; $\pas_{t,0}
    \geq 0$, $\pas_{t,k} >0$ for $t \in [\kouter]^\star$, $k \in
         [\xi_t]^\star$.}  \KwResult{The SPIDER-EM sequence:
    $\{\hatS_{t,k}\}$} {$\hatS_{1,0} = \hatS_{1,-1} = \hatS_\init$ \;
    $\Smem_{1,0} = \bars \circ \map(\hatS_{1,-1}) +
    \mathcal{E}_1$ \label{eq:SA:reset:0}\; \For{$t=1, \cdots,
      \kouter$ \label{eq:SA:epoch}}{ \For{$k=0, \ldots,\xi_t-1$} {
        Sample a mini batch $\batch_{t,k+1}$ of size $\lbatch$ in
        $[n]^\star$ \label{eq:SA:update:batch} \; $\Smem_{t,k+1} =
        \Smem_{t,k} + \lbatch^{-1} \sum_{i \in
          \batch_{t,k+1}} \hspace{-0.15cm} \left( \bars_i \circ
        \map(\hatS_{t,k}) - \bars_i \circ \map(\hatS_{t,k-1})
        \right)$ \label{eq:SA:update:Smem} \; $\hatS_{t,k+1} =
        \hatS_{t,k} + \pas_{t,k+1} \left( \Smem_{t,k+1} - \hatS_{t,k}
        \right)$ \label{eq:SA:updateclassical}} $\hatS_{t+1,-1} =
      \hatS_{t,\xi_t}$ \; $\Smem_{t+1,0} = \bars \circ
      \map(\hatS_{t+1,-1}) + \mathcal{E}_{t+1}$ \label{eq:SA:reset:1}
      \; $\hatS_{t+1,0} = \hatS_{t+1,-1} + \pas_{t+1,0} \left(
      \Smem_{t+1,0} - \hatS_{t+1,-1} \right)$ } }
    \caption{The {\tt g-SPIDER-EM} algorithm. The $\mathcal{E}_t$'s are
      introduced as a perturbation to the computation of $\bars \circ
      \map(\hatS_{t,-1})$; they can be null.\label{algo:SPIDERSA}}
\end{algorithm}
\vspace{-0.3cm} 

The algorithm {\em generalized Stochastic
  Path-Integrated Differential EstimatoR Expectation Maximization}
({\tt g-SPIDER-EM}) described by Algorithm~\ref{algo:SPIDERSA}
uses a new strategy when defining the approximation of $\bars
\circ \map(s)$ at each iteration. It is composed of nested loops:
$\kouter$ outer loops, each of them formed with a possibly random
number of inner loops. Within the $t$th outer loop, {\tt g-SPIDER-EM}
mimics the identity $\bars \circ \map(\hatS_{t,k}) = \bars \circ
\map(\hatS_{t,k-1}) + \{\bars \circ \map(\hatS_{t,k}) - \bars \circ
\map(\hatS_{t,k-1}) \}$.  More precisely, at iteration $k+1$, the
approximation $\Smem_{t,k+1}$ of the full sum $\bars \circ
\map(\hatS_{t,k})$ is the sum of the current approximation
$\Smem_{t,k}$ and of a Monte Carlo approximation of the difference
(see Lines~\ref{eq:SA:update:batch}, \ref{eq:SA:update:Smem}, in
Algorithm~\ref{algo:SPIDERSA}); the examples $i$ in $\batch_{t,k+1}$
used in the approximation of $\bars \circ \map(\hatS_{t,k})$ and those
used for the approximation of $\bars \circ \map(\hatS_{t,k-1})$ are
the same - which make the approximations correlated and favor a
variance reduction when plugged in the SA update
(Line~\ref{eq:SA:updateclassical}). $\batch_{t,k+1}$ is sampled with
or without replacement; even when $\batch_{t,k+1}$ collects
independent examples sampled uniformly in $[n]^\star$, we have
$\PE\left[ \Smem_{t,k+1} \vert \F_{t,k} \right] - \bars \circ
\map(\hatS_{t,k}) = \Smem_{t,k} - \bars \circ \map(\hatS_{t,k-1})$
where $\F_{t,k}$ is the sigma-field collecting the randomness up to
the end of the outer loop $\# t$ and inner loop $\# k$: the
approximation $\Smem_{t,k+1}$ of $\bars \circ \map(\hatS_{t,k})$ is
biased - a property which makes the theoretical analysis of the
algorithm challenging. This approximation is reset (see
Lines~\ref{eq:SA:reset:0},\ref{eq:SA:reset:1}) at the end of an outer
loop: in the "standard" {\tt SPIDER-EM}, $\Smem_{t,0} = \bars \circ
\map(\hatS_{t,-1})$ is computed, but this "refresh" can be only
partial, by computing an update on a (large) batch $\tilde
\batch_{t,0}$ (size $\tilde \lbatch_t$) of observations: $\Smem_{t,0}
= {\tilde \lbatch}_t^{-1} \sum_{i \in \tilde \batch_{t,0}} \bars_i
\circ \map(\hatS_{t,-1})$.  Such a reset starts a so-called {\em
  epoch} (see Line~\ref{eq:SA:epoch}). The number of inner loops
$\xi_t$ at epoch $\# t$ can be deterministic $\xi_t$; or random, such
as a uniform distribution on $[\kin]^\star$ or a geometric
distribution, and drawn prior the run of the algorithm.

Comparing {\tt g-SPIDER-EM} with {\tt
SPIDER-EM}~\cite{fort:moulines:wai:2020}, we notice that the former allows 
a perturbation $\mathcal{E}_t$ when initializing $\Smem_{t,0}$. This is important
for computational cost reduction. Moreover, {\tt g-SPIDER-EM} considers epochs with
time-varying length $\xi_t$ which covers situations when it is random
and chosen independently of the other sources of randomness (the
errors $\mathcal{E}_t$, the batches $\batch_{t,k+1}$).  Hereafter, we
provide an original analysis of an {\tt g-SPIDER-EM}, namely {\tt
  Geom-SPIDER-EM} which corresponds to the case $\xi_t \leftarrow
\Xi_t$, $\Xi_t$ being a geometric r.v. on $\nset^\star$ with success
probability $1-\rho_t \in \ooint{0,1}$: $\PP(\Xi_t = k) = (1-\rho_t)
\rho_t^{k-1}$ for $k \geq 1$ (hereafter, we will write $\Xi_t \sim
\mathcal{G}^\star(1-{\rho_t})$). Since $\Xi_t$ is also the first
success distribution in a sequence of independent Bernoulli trials,
the geometric length could be replaced with: \textit{(i)} at each
iteration $k$ of epoch $t$, sample a Bernoulli r.v. with a probability
of success $(1-\rho_t)$; \textit{(ii)} when the coin comes up head,
start a new epoch (see ~\cite{PAGE:2020,geomSARAH-2020} for similar
ideas on stochastic gradient algorithms).

Let us establish complexity bounds for {\tt
  Geom-SPIDER-EM}. We analyze a randomized terminating iteration
$\Xi^\star$ \cite{ghadimi:lan:2013} and discuss how to choose $\kouter, \lbatch$ and $\xi_1,
\cdots, \xi_{\kouter}$ as a function of the batch size $n$ and an
accuracy $\epsilon >0$ to reach $\epsilon$-approximate
stationarity i.e. $\PE [\|h(\hatS_{\Xi^\star})\|^2 ] \leq \epsilon$.
To this end, we endow the probability space $(\Omega, \mathcal{A}, \PP)$ with
 the sigma-fields $\F_{1,0} \eqdef \sigma( \mathcal{E}_1)$, $\F_{t,0}
 \eqdef \sigma(\F_{t-1, \xi_t}, \mathcal{E}_t)$ for $t \geq 2$, and
 $\F_{t,k+1} \eqdef \sigma(\F_{t,k}, \batch_{t,k+1})$ for $t \in
        [\kouter]^\star, k \in [\xi_t-1]$. For a r.v. $\Xi_t \sim
        \mathcal{G}^\star(1-\rho_t)$, set $\bPE_t[\phi(\Xi_t) \vert
          \F_{t,0}] \eqdef (1-\rho_t) \sum_{k \geq 1} \rho^{k-1}_t
        \PE[\phi(k) \vert \F_{t,0}]$ for any bounded measurable
        function $\phi$.
\vspace{-0.2cm}
\begin{theorem}
  \label{theo:main}
  Assume A\ref{hyp:model} to A\ref{hyp:regV}. For any $t \in
  [\kouter]^\star$, let $\rho_{t} \in \ooint{0,1}$ and $\Xi_t \sim
  \mathcal{G}^\star(1-{\rho_t})$. Run Algorithm~\ref{algo:SPIDERSA}
  with $\pas_{t,k+1}= \pas_t>0$ and $\xi_t \leftarrow \Xi_t$ for any
  $t \in [\kouter]^\star$, $k \geq 0$. Then, for any $t \in
  [\kouter]^\star$, \begin{align*} & \frac{v_{\min}
      \pas_t}{2(1-\rho_t)} \bPE_t\left[ \|h(\hatS_{t,\Xi_t-1})\|^2
      \vert \F_{t,0} \right] \\ & \leq \lyap(\hatS_{t,0}) -
    \bPE_t\left[\lyap(\hatS_{t,\Xi_t}) \vert \F_{t,0} \right] +
    \frac{v_\max \pas_{t} }{2(1-\rho_t)}  \|\mathcal{E}_t\|^2 \\ &+
    \frac{v_\max \pas_{t} \pas_{t,0}^2  }{2(1-\rho_t)} \frac{L^2 }{\lbatch} 
    \| \Delta \hatS_{t,0} \|^2  + \mathcal{N}_t \,
    \bPE_t\left[ \| \Delta \hatS_{t,\Xi} \|^2 \vert
      \F_{t,0} \right] \eqsp;
  \end{align*}
  where $\Delta \hatS_{t,\xi} \eqdef \Smem_{t,\xi} - \hatS_{t,\xi-1}$, $L^2 \eqdef n^{-1} \sum_{i=1}^n L_i$, and
  \[
\mathcal{N}_t \eqdef - \frac{\pas_{t}}{2(1-\rho_t)} \left( v_{\min} -
\pas_{t} L_{\dot \lyap} - \frac{v_\max L^2 \rho_t }{(1-\rho_t)
  \lbatch}\pas_t^2\right) \eqsp.
  \]
 \end{theorem}
\noindent Theorem~\ref{theo:main} is the key result from which our conclusions
 are drawn; its proof is adapted from \cite[section
   8]{fort:moulines:wai:2020} (also see~\cite{fort:moulines:wai:icassp:supp}).

 Let us discuss the rate of convergence and the complexity of {\tt
 Geom-SPIDER-SA} in the case: for any $t \in [\kouter]^\star$, the
 mean number of inner loops is $(1-\rho_t)^{-1} = \kin$, $\pas_{t,0} =
 0$ and $\pas_t = \alpha/L$ for $\alpha > 0$ satisfying
\[
v_{\min} - \alpha \frac{L_{\dot  \lyap}}{L} - \alpha^2 v_\max
\frac{\kin}{\lbatch} \left(1 - \frac{1}{\kin} \right) >0 \eqsp.
\] {\bf Linear rate.} When
$\Xi \sim \mathcal{G}^\star(1-{\rho})$, we have
\begin{equation} \label{eq:key:geom} \rho \PE\left[D_\Xi \right] \leq \rho \PE\left[D_\Xi \right]
  +(1-\rho) D_0 = \PE\left[D_{\Xi-1} \right]
\end{equation}
for any positive sequence $\{D_k, k \geq 0\}$, Theorem~\ref{theo:main}
implies
\begin{multline}
  \bPE_t\left[ \|h(\hatS_{t,\Xi_t})\|^2 \vert \F_{t,0} \right] \leq
  \frac{2L}{v_{\min} \alpha (\kin-1)} \left(\lyap(\hatS_{t,0}) - \min
    \lyap \right)  \\
  + \frac{v_\max}{v_{\min}} \frac{\kin}{\kin-1} \|\mathcal{E}_t\|^2
  \eqsp. \label{eq:prepareLR}
\end{multline}
Hence, when $\|\mathcal{E}_t\| = 0$ and $\lyap$ satisfies a
Polyak-{\L}ojasiewicz condition \cite{karimi2016linear}, \ie\ 
\begin{equation}\label{eq:PL}
\exists \tau>0, \forall s \in \rset^q, \qquad \lyap(s) - \min \lyap \leq \tau \|
\nabla \lyap(s) \|^2
\end{equation}
then \eqref{eq:prepareLR} yields
\[
\mathcal{H}_t \eqdef \bPE_t\left[ \|h(\hatS_{t,\Xi_t})\|^2 \vert
  \F_{t,0} \right] \leq \frac{2L\tau v_\max^2 }{v_{\min} \alpha
  (\kin-1)} \| h(\hatS_{t-1, \xi_{t-1}}) \|^2 \eqsp,
\]
thus establishing a linear rate of the algorithm along the path
$\{\hatS_{t,\Xi_t}, t \in [\kouter]^\star \}$ as soon as $\kin$ is
large enough:
\[
 \PE\left[\mathcal{H}_t \right] \leq \left( \frac{2L \tau v_\max^2
 }{v_{\min} \alpha (\kin-1)}\right)^t \| h(\hatS_{\init}) \|^2.
\]
Even if the Polyak-{\L}ojasiewicz condition \eqref{eq:PL} is quite
restrictive, the above discussion gives the intuition of the {\em
  lock-in} phenomenon which often happens at convergence: a linear
rate of convergence is observed when the path is trapped in a
neighborhood of its limiting point, which may be the consequence that
locally, the Polyak-{\L}ojasiewicz condition holds (see figure~\ref{fig:1} in
Section~\ref{sec:numerical}).\vspace{.1cm}

\noindent {\bf Complexity for $\epsilon$-approximate stationarity.}
From Theorem~\ref{theo:main}, Eq.~\eqref{eq:key:geom} and $\hatS_{t,
  \Xi_t} = \hatS_{t+1,0}$ (here $\pas_{t,0} = 0$), it holds
\[
\frac{v_{\min} \alpha (\kin-1) }{2L} \PE\left[ \mathcal{H}_t\right]
\leq \PE\left[ \lyap(\hatS_{t,0}) - \lyap(\hatS_{t+1,0}) \right]
\eqsp.
\]
Therefore,
\begin{equation}\label{eq:random:termination}
\frac{1}{\kouter} \sum_{t=1}^{\kouter} \PE\left[\mathcal{H}_t\right] \leq \frac{2L\left(
  \lyap(\hatS_\init) - \min \lyap \right)}{v_{\min} \alpha (\kin-1)
  \kouter} \eqsp.
\end{equation}
Eq.~\eqref{eq:random:termination} establishes that in order to obtain
an $\epsilon$-approximate stationary point, it is sufficient to stop
the algorithm at the end of the epoch $\# T$, where $T$ is sampled
uniformly in $[\kouter]^\star$ with
$\kouter = O( \epsilon^{-1} L/\kin)$ - and return
$\hatS_{T,\xi_T}$. To do such, the mean number of conditional
expectations evaluations is
$\mathcal{K}_{\mathrm{CE}} \eqdef n+n \kouter+ 2 \lbatch \kin
\kouter$;
and the mean number of optimization steps is
$\mathcal{K}_{\mathrm{Opt}} \eqdef \kouter + \kin \kouter$. By
choosing $\kin = O(\sqrt{n})$ and $\lbatch = O(\sqrt{n})$, we have
$\mathcal{K}_{\mathrm{CE}} = O(L \sqrt{n} \epsilon^{-1})$ and
$\mathcal{K}_{\mathrm{Opt}} = O(L \epsilon^{-1})$.  Similar randomized
terminating strategies were proposed in the literature: their optimal
complexity in terms of conditional expectations evaluations is
$O(\epsilon^{-2})$ for {\tt Online-EM}~\cite{cappe:moulines:2009},
$O(\epsilon^{-1} n)$ for {\tt i-EM}~\cite{Neal:hinton:1998},
$O(\epsilon^{-1} n^{2/3})$ for {\tt
  sEM-vr}~\cite{chen:etal:2018,karimi:etal:2019},
$O(\{\epsilon^{-1} n^{2/3}\} \wedge \{\epsilon^{-3/2} \sqrt{n} \})$
for {\tt FIEM}~\cite{karimi:etal:2019,fort:gach:moulines:2020} and
$O(\epsilon^{-1} \sqrt{n})$ for {\tt SPIDER-EM} - see \cite[section
6]{fort:moulines:wai:2020} for a comparison of the complexities
$\mathcal{K}_{\mathrm{CE}}$ and $\mathcal{K}_{\mathrm{opt}}$ of these
incremental EM algorithms. Hence, {\tt Geom-SPIDER-EM} has the same
complexity bounds as {\tt SPIDER-EM}, and they are optimal among the
class of incremental EM algorithms.

\section{Numerical illustration}
\label{sec:numerical}
We perform experiments on the MNIST dataset, which consists of
$n=6 \times 10^4$ images of handwritten digits, each with $784$ pixels. We
pre-process the datas as detailed in \cite[section
5]{fort:gach:moulines:2020}: $67$ uninformative pixels are removed
from each image, and then a principal component analysis is applied to
further reduce the dimension; we keep the $20$ principal components of
each observation. The learning problem consists in fitting a Gaussian
mixture model with $g=12$ components: $\param$ collects the weights of
the mixture, the expectations of the components (i.e. $g$ vectors in
$\rset^{20}$) and a full  $20 \times 20$ covariance matrix;
here, $\R = 0$ (no penalty term). All the algorithms start from an
initial value $\hatS_\init = \bars \circ \map(\param_\init)$ such that
$-F(\param_\init) = - 58.3$, and their first two epochs are {\tt
  Online-EM}. The first epoch with a variance reduction technique is
epoch $\# 3$; on Fig.~\ref{fig:1}, the plot starts at epoch $\# 2$.

The proposed {\tt Geom-SPIDER-EM} is run with a constant step size
$\pas_{t,k} = 0.01$ (and $\pas_{t,0} =0$); $\kouter =148$ epochs
(which are preceded with $2$ epochs of {\tt Online-EM}); a mini batch
size $\lbatch = \sqrt{n}$. Different strategies are considered for the
initialization $\Smem_{t,0}$ and the parameter of the geometric
r.v. $\Xi_t$.  In {\tt full-geom}, $\kin = \sqrt{n}/2$ so that the
mean total number of conditional expectations evaluations per outer
loop is $2 \lbatch \kin =n$; and $\mathcal{E}_t =0$ which means that
$\Smem_{t,0}$ requires the computation of the full sum $\bars$ over
$n$ terms. In {\tt half-geom}, $\kin$ is defined as in {\tt
  full-geom}, but for all $t \in [\kouter]^\star$,
$\Smem_{t,0} = (2/n) \sum_{i \in \tilde \batch_{t,0}} \bars_i \circ
\map(\hatS_{t,-1})$
where $\batch_{t,0}$ is of cardinality $n/2$; therefore
$\mathcal{E}_t \neq 0$.  In {\tt quad-geom}, a quadratic growth is
considered both for the mean of the geometric random variables:
$\PE\left[\xi_t\right] =\mathrm{min}(n, \mathrm{max}(20 t^2,
n/50))/(2\lbatch)$;
and for the size of the mini batch when computing $\Smem_{t,0}$:
$\Smem_{t,0}=\tilde \lbatch_{t}^{-1} \sum_{i \in \tilde \batch_{t}}
\bars_i \circ \map(\hatS_{t,-1})$
with $\tilde \lbatch_{t} = \mathrm{min}(n, \mathrm{max}(20 t^2, n/50))$.
The {\tt g-SPIDER-EM} with a constant number of inner loops
$\xi_t = \kin = n/(2 \lbatch)$ is also run for comparison: different
strategies for $\Smem_{t,0}$ are considered, the same as above (it
corresponds to {\tt full-ctt}, {\tt half-ctt} and {\tt quad-ctt} on
the plots). Finally, in order to illustrate the benefit of the variance
reduction, a pure {\tt Online-EM} is run for $150$ epochs, one epoch
corresponding to $\sqrt{n}$ updates of the statistics $\hatS$, each of
them requiring a mini batch $\batch_{k+1}$ of size $\sqrt{n}$ (see
Eq.\eqref{eq:SA:scheme}).

The algorithms are compared through an estimation of the quantile of order $0.5$ of
$\|h(\hatS_{t,\Xi_t})\|^2$ over $30$
independent realizations. It is plotted versus the number of epochs
$t$ in Fig.~\ref{fig:1} and the number of conditional expectations (CE)
evaluations in Fig.~\ref{fig:2}. They are also compared through the
objective function $F$ along the path; the mean value over $30$
independent paths is displayed versus the number of CE, see Fig.~\ref{fig:3}.  

We first observe that {\tt Online-EM} has a poor convergence rate, thus
justifying the interest of variance reduction techniques as shown in
Fig.~\ref{fig:1}. Having a persistent bias along iterations when defining
$\Smem_{t,0}$ i.e.  considering $\tilde \lbatch_{t} \neq n$ and
therefore $\mathcal{E}_t \neq 0$, is also a bad strategy as seen in Fig.~\ref{fig:1}, \ref{fig:2}
for {\tt half-ctt} and {\tt half-geom}. For the four other
{\tt SPIDER-EM} strategies, we observe a linear convergence rate in Fig.~\ref{fig:1}, \ref{fig:2}. 
The best strategy, both in terms of CE evaluations and
in terms of efficiency given a number of epochs, is {\tt quad-ctt}: a
constant and deterministic number of inner loops $\xi_t$ combined with
an increasing accuracy when computing $\Smem_{t,0}$; therefore, during
the first iterations, it is better to reduce the computational cost
of the algorithm by considering $\tilde \lbatch_{t} \ll n$. When
$\mathcal{E}_t =0$ (i.e. $\tilde \lbatch_{t} =n$ so the computational
cost of $\Smem_{t,0}$ is maximal), it is possible to reduce the total
CE computational cost of the algorithm by considering a random number
of inner loops (see {\tt full-geom} and {\tt full-ctt} on
Fig.~\ref{fig:1}, \ref{fig:2}). Finally, the strategy which consists in increasing both
$\tilde \lbatch_{t}$ and the number of inner loops, does not look the
best one (see {\tt quad-ctt} and {\tt quad-geom} on Fig.~\ref{fig:1} to Fig.~\ref{fig:3}).

\begin{figure}[t]
  \centering
  {\includegraphics[width=7cm]{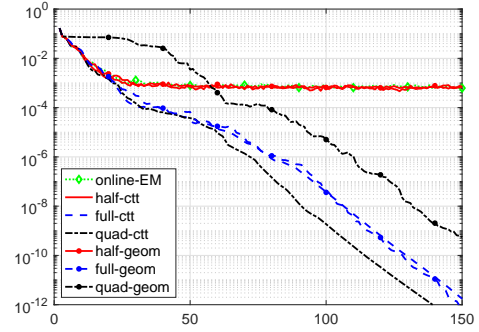}}\vspace{-.2cm}
  \caption{Quantile $0.5$ of $\|h(\hatS_{t,\Xi_t})\|^2$ vs the number of epochs} \label{fig:1}
\end{figure}

\begin{figure}[t]
  \centering
  \centerline{\includegraphics[width=7cm]{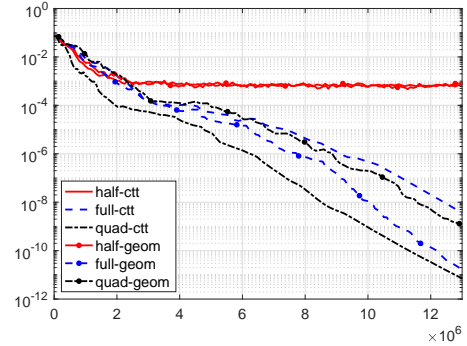}}\vspace{-.2cm}
  \caption{Quantile $0.5$ of $\|h(\hatS_{t,\Xi_t})\|^2$ vs the
    number of CE evaluations} \label{fig:2}
\end{figure}

\begin{figure}[b]
  \centering
  {\includegraphics[width=4cm]{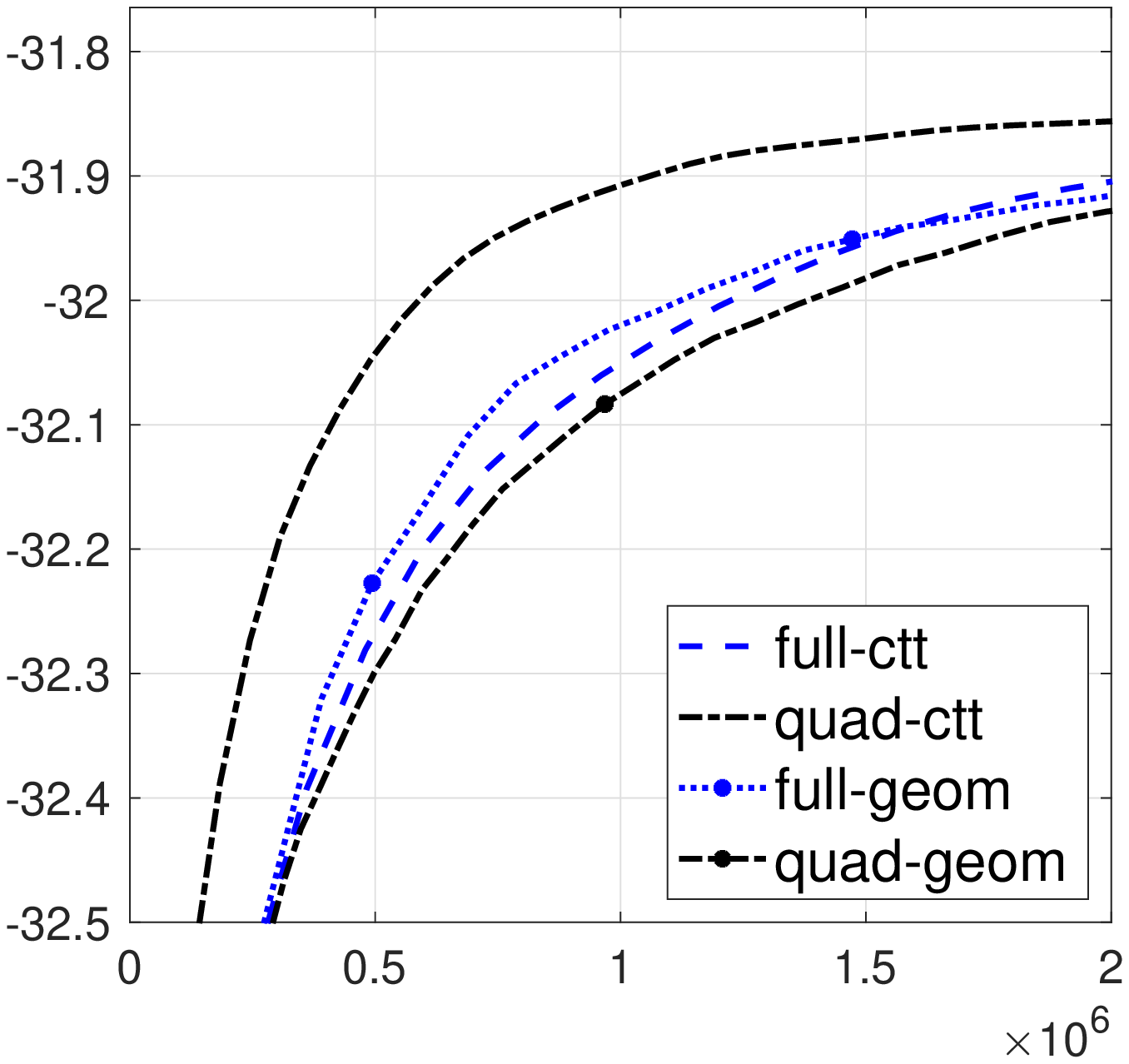}}
  {\includegraphics[width=4cm]{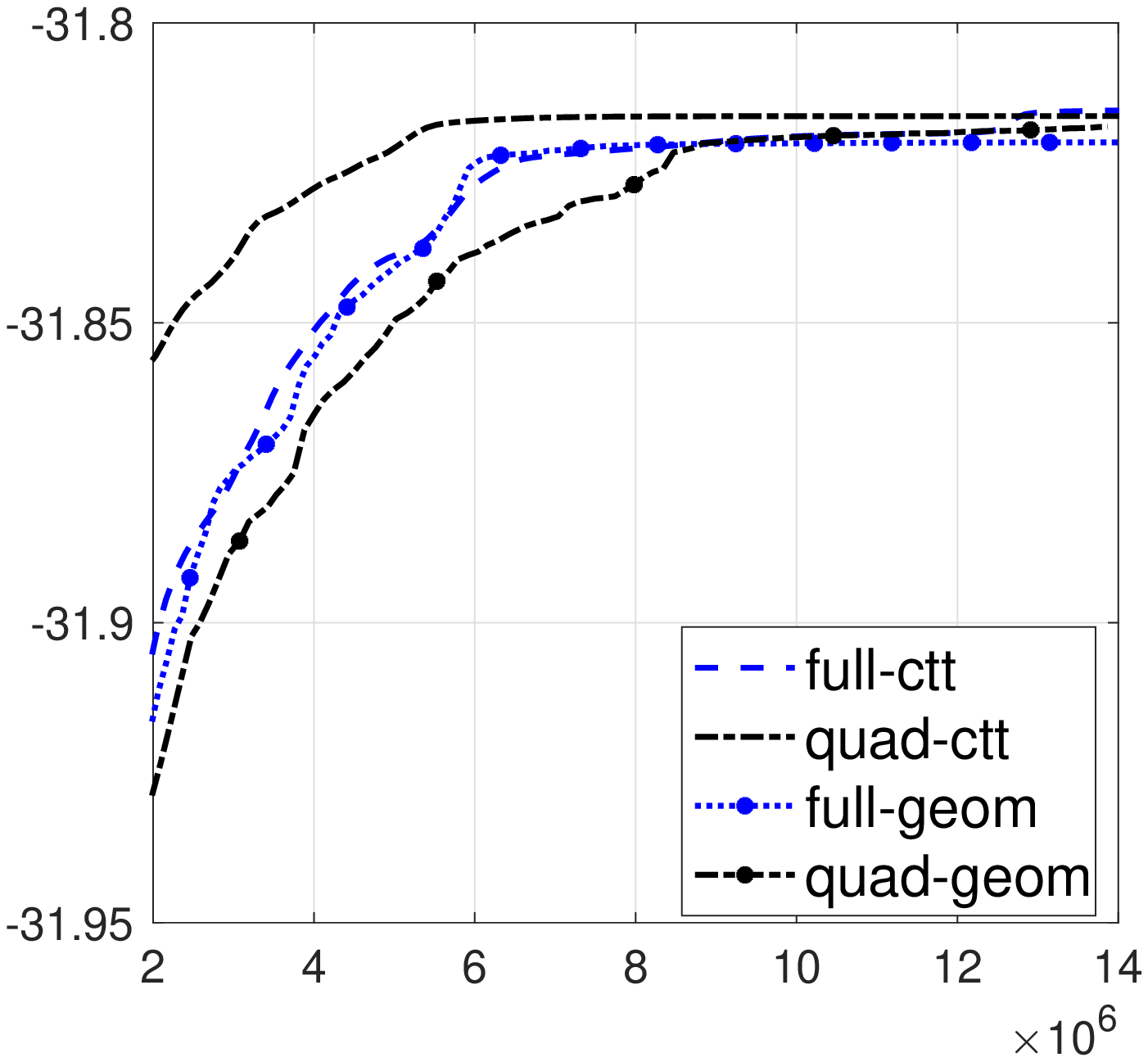}}
  \caption{(Left) $-F$ vs CE, until $2e6$. (Right) $-F$ vs CE, after $2e6$.} \label{fig:3}
\end{figure}

\vfill
\pagebreak

\clearpage
\newpage

\begin{center}
 {\bf \large Supplementary material,  \\
paper ``GEOM-SPIDER-EM: faster variance reduced stochastic Expectation
Maximization for nonconvex finite-sum optimization''}
\end{center}

\section{Proof of Theorem 1}
 Let $\{\mathcal{E}_t, t \in [\kouter]^\star\}$ and $\{\batch_{t,k+1},
 t \in [\kouter]^\star, k \in [\xi_t-1]\}$ be random variables defined
 on the probability space $(\Omega, \mathcal{A}, \PP)$.  Define the
 filtrations $\F_{1,0} \eqdef \sigma( \mathcal{E}_1)$, $\F_{t,0}
 \eqdef \sigma(\F_{t-1, \xi_t}, \mathcal{E}_t)$ for $t \geq 2$, and
 $\F_{t,k+1} \eqdef \sigma(\F_{t,k}, \batch_{t,k+1})$ for $t \in
        [\kouter]^\star, k \in [\xi_t-1]$.

For $\rho_{t} \in \ooint{0,1}$, set
\[
\bPE_t[\phi(\Xi_t) \vert \F_{t,0}] \eqdef (1-\rho_t) \sum_{k \geq 1}
\rho^{k-1}_t \PE[\phi(k) \vert
\F_{t,0}] \eqsp,
\]
for any measurable positive function $\phi$.

\begin{assumption} \label{hyp:model} $\Param \subseteq \rset^d$ is a
   convex set.  $(\Zset, \Zsigma)$ is a measurable space and $\mu$ is
   a $\sigma$-finite positive measure on $\Zsigma$. The functions $\R:
   \Param \to \rset$, $\phi : \Param \to \rset^q$, $\psi: \Param \to
   \rset$, $\s_i: \Zset \to \rset^q$, $h_i: \Zset \to \rset_+$ for all
   $i \in [n]^\star$ are measurable. For any $\param \in \Param$ and
   $i \in [n]^\star$, $|\loss{i}(\param)| < \infty$.
\end{assumption}

\begin{assumption} \label{hyp:bars} The expectations $\bars_i(\param)$
  are well defined for all $\param \in \Param$ and $i \in [n]^\star$.
  For any $s \in \rset^q$, $ \argmin_{\param \in \Param} \ \left(
  \psi(\param) - \pscal{s}{\phi(\param)} + \R(\param) \right)$ is a
  (non empty) singleton denoted by $\{\map(s)\}$.
\end{assumption}

\begin{assumption} \label{hyp:regV} The functions $\phi$, $\psi$ and
  $\R$ are continuously differentiable on $\Param^v$, where $\Param^v$
  is a neighborhood of $\Param$.  $\map$ is continuously
  differentiable on $\rset^q$. The function $F$ is continuously
  differentiable on $\Param^v$ and for any $\param \in \Param$,
  $\nabla F(\param) = - \nabla \phi(\param)^T \bars(\param) + \nabla
  \psi(\param) + \nabla \R(\param)$. For any $s \in \rset^q$, $B(s)
  \eqdef \nabla{\left(\phi \circ \map \right)}(s)$ is a symmetric $q
  \times q$ matrix and there exist $0< v_{min} \leq v_{max}< \infty $
  such that for all $s\in \rset^q$, the spectrum of $B(s)$ is in
  $\ccint{v_{min}, v_{max}}$.  For any $i \in [n]^\star$, $\bars_i
  \circ \map$ is globally Lipschitz on $\rset^q$ with constant $L_i$.
  The function $ s \mapsto \nabla (F \circ \map)(s) = B(s) \left(
  \bars \circ T(s) -s \right)$ is globally Lipschitz on $\rset^q$ with
  constant $L_{\dot \lyap}$.
\end{assumption}

\begin{lemma}
  \label{lem:stopping:geom}
  Let $\rho \in \ooint{0,1}$ and $\{D_k, k \geq 0 \}$ be real numbers
  such that $\sum_{k \geq 0} \rho^k |D_k| < \infty$. Let
  $\xi\sim\mathcal{G}^\star(1-{\rho})$. Then
  $\PE[D_{\xi-1}] = \rho \PE[D_{\xi}] + (1-\rho) D_0 = \PE[D_{\xi}] +
  (1-\rho) (D_0 - \PE[D_\xi])$.
\end{lemma}
\begin{proof}By definition of $\xi$,
  \begin{align*}
    \PE\left[D_\xi \right] &= (1-\rho) \sum_{k \geq 1} \rho^{k-1} D_k \\
    & = \rho^{-1} (1-\rho) \sum_{k \geq 1} \rho^{k} D_k \\
    & = \rho^{-1} (1-\rho) \sum_{k \geq 0} \rho^{k} D_k  - \rho^{-1} (1- \rho) D_0\\
    & = \rho^{-1} (1-\rho) \sum_{k \geq 1} \rho^{k-1} D_{k-1}  - \rho^{-1} (1- \rho) D_0\\
  &=  \rho^{-1} \PE\left[D_{\xi-1}\right]  - \rho^{-1} (1- \rho) D_0 \eqsp.
  \end{align*}
  This yields $\rho \PE\left[D_\xi \right] = \PE\left[D_{\xi-1}\right]
  - (1- \rho) D_0$ and concludes the proof.
\end{proof}

\begin{lemma}
  \label{lem:batch}
  For any $t \in [\kouter]^\star$, $k \in [\xi_t]^\star$, $\batch_{t,k}$ and
  $\F_{t,k-1}$ are independent. In addition, for any $s \in \rset^q$,
  $\lbatch^{-1} \PE\left[ \sum_{i \in \batch_{t,k}} \bars_i \circ
    \map(s) \right] = \bars \circ \map(s)$. Finally, assume that for
  any $i \in [n]^\star$, $\bars_i \circ \map$ is globally Lipschitz
  with constant $L_i$. Then for any $s,s' \in \rset^q$,
  \begin{multline*}
\hspace{-0.4cm}\PE\left[\|\lbatch^{-1} \sum_{i \in \batch_{t,k}} \left\{ \bars_i
  \circ \map(s) - \bars_i \circ \map(s') \right\} - \bars \circ \map(s) +
  \bars \circ \map(s') \|^2  \right] \\
\leq \frac{1}{\lbatch}\left(L^2 \|s-s'\|^2 - \| \bars \circ \map(s) - \bars \circ \map(s')\|^2 \right) \eqsp,
  \end{multline*}
  where $L^2 \eqdef n^{-1} \sum_{i=1}^n L_i^2$.
\end{lemma}
\begin{proof}
  See \cite[Lemma 4]{fort:moulines:wai:2020}; the proof holds true
  when $\batch_{t,k}$ is sampled with or without replacement.
  \end{proof}

\begin{proposition}
  \label{prop:bias}
  For any $t \in [\kouter]^\star$, $k \in [\xi_t-1]$,
  \[
\PE\left[ \Smem_{t,k+1} \vert \F_{t,k} \right] - \bars \circ
\map(\hatS_{t,k}) = \Smem_{t,k} - \bars \circ \map(\hatS_{t,k-1})
\eqsp,
\]
and
\[
\PE\left[ \Smem_{t,k+1} - \bars \circ \map(\hatS_{t,k}) \vert \F_{t,0}
  \right] = \mathcal{E}_t \eqsp.
\]
  \end{proposition}
\begin{proof}
Let $t \in [\kouter]^\star$, $k \in [\xi_t-1]$.  By
Lemma~\ref{lem:batch},
  \begin{align*}
    \PE\left[ \Smem_{t,k+1} \vert \F_{t,k} \right] = \Smem_{t,k} +
    \bars \circ \map(\hatS_{t,k}) - \bars \circ \map(\hatS_{t,k-1}) \eqsp.
  \end{align*}
 By definition of $\Smem_{t,0}$ and of the filtrations, $\Smem_{t,0} -
 \bars \circ \map(\hatS_{t,-1}) = \mathcal{E}_t \in \F_{t,0}$. The
 proof follows by induction on $k$.
\end{proof}

\begin{proposition}
  \label{prop:variance}
 Assume that for any $i \in [n]^\star$, $\bars_i \circ \map$ is
 globally Lipschitz with constant $L_i$.  For any $t \in
 [\kouter]^\star$, $k \in [\xi_t-1]$,
  \begin{align*}
 &\PE\left[ \|\Smem_{t,k+1} - \PE\left[ \Smem_{t,k+1} \vert \F_{t,k}
        \right] \|^2 \vert \F_{t,k} \right] \\ & \leq
    \frac{1}{\lbatch}\left(L^2 \|\hatS_{t,k} - \hatS_{t,k-1}\|^2 - \|
    \bars \circ \map(\hatS_{t,k}) - \bars \circ
    \map(\hatS_{t,k-1})\|^2 \right) \\ & \leq \frac{L^2}{\lbatch}
    \pas_{t,k}^2 \|\Smem_{t,k} - \hatS_{t,k-1}\|^2 \eqsp,
  \end{align*}
where $L^2 \eqdef n^{-1} \sum_{i=1}^n L_i^2$.
  \end{proposition}
\begin{proof}
  Let $t \in [\kouter]^\star$, $k \in [\xi_t-1]$.  By
  Lemma~\ref{lem:batch}, Proposition~\ref{prop:bias}, the definition
  of $\Smem_{t,k+1}$ and of the filtration $\F_{t,k}$,
  \begin{align*}
    & \Smem_{t,k+1} - \PE\left[ \Smem_{t,k+1} \vert \F_{t,k} \right]
    \\
    & = \Smem_{t,k+1} - \bars \circ \map(\hatS_{t,k}) - \Smem_{t,k}
    + \bars \circ \map(\hatS_{t,k-1})
    \\ & = \lbatch^{-1} \sum_{i \in
      \batch_{t,k+1}} \{ \bars_i \circ \map(\hatS_{t,k}) - \bars_i
    \circ \map(\hatS_{t,k-1}) \} \\
    & - \bars \circ \map(\hatS_{t,k}) + \bars
    \circ \map(\hatS_{t,k-1}) \eqsp.
  \end{align*}
  We then conclude by Lemma~\ref{lem:batch} for the first inequality;
  and by using the definition of $\hatS_{t,k}$ for the second one.
\end{proof}

\begin{proposition}
  \label{prop:squarederror}
  Assume that for any $i \in [n]^\star$, $\bars_i \circ \map$ is
 globally Lipschitz with constant $L_i$.  For any $t \in
 [\kouter]^\star$, $k \in [\xi_t-1]$,
  \begin{align*}
 &\PE\left[ \|\Smem_{t,k+1} - \bars \circ \map(\hatS_{t,k}) \|^2 \vert
      \F_{t,k} \right] \\ & \leq \frac{L^2}{\lbatch} \pas_{t,k}^2
    \|\Smem_{t,k} - \hatS_{t,k-1}\|^2 +\|\Smem_{t,k} - \bars \circ
    \map(\hatS_{t,k-1})\|^2 \eqsp,
  \end{align*}
 where $L^2 \eqdef n^{-1} \sum_{i=1}^n L_i^2$.
  \end{proposition}
\begin{proof}
  By definition of the conditional expectation, we have for any
  r.v. $\phi(V)$
  \[
  \PE\left[ \| U - \phi(V) \|^2 \vert V \right] = \PE\left[ \| U -
    \PE\left[U \vert V \right] \|^2 \vert V \right] + \| \PE\left[U
    \vert V \right] - \phi(V) \|^2 \eqsp.
  \]
  The proof follows from this equality and Propositions~\ref{prop:bias}
  and \ref{prop:variance}.
\end{proof}
\begin{corollary}
  \label{coro:squarederror:geom}
  Assume that for any $i \in [n]^\star$, $\bars_i \circ \map$ is
  globally Lipschitz with constant $L_i$. For any
  $t \in [\kouter]^\star$, let $\rho_{t} \in \ooint{0,1}$ and
  $\Xi_t\sim\mathcal{G}^\star(1-{\rho_t})$. For any
  $t \in [\kouter]^\star$,
 \begin{align*}
&\bPE_t\left[ \left( \pas_{t,\Xi_t} - \rho_t \pas_{t,\Xi_t+1} \right)
     \|\Smem_{t,\Xi_t} - \bars \circ \map(\hatS_{t,\Xi_t-1}) \|^2
     \vert \F_{t,0} \right] \\ & \leq \frac{L^2 \rho_t}{\lbatch}
   \bPE_t\left[\pas_{t,\Xi_t+1} \pas_{t,\Xi_t}^2 \|\Smem_{t,\Xi_t} -
     \hatS_{t,\Xi_t-1}\|^2 \vert \F_{t,0} \right] \\ & + \frac{L^2
     (1-\rho_t)}{\lbatch} \pas_{t,1} \pas_{t,0}^2 \|\Smem_{t,0} -
   \hatS_{t,-1}\|^2 + (1-\rho_t) \pas_{t,1} \|\mathcal{E}_t\|^2 \eqsp,
   \end{align*}
 where $L^2 \eqdef n^{-1} \sum_{i=1}^n L_i^2$.
\end{corollary}
\begin{proof}
  Let $t \in [\kouter]^\star$ and $k \in [\xi_t-1]$. From
  Proposition~\ref{prop:squarederror} and since $\F_{t,0} \subseteq
  \F_{t,k}$ for $k \in [\xi_t-1]$, we have
  \begin{align*}
& \PE\left[\|\Smem_{t,k+1} - \bars \circ \map(\hatS_{t,k}) \|^2 \vert
      \F_{t,0} \right] \\ & \leq \PE\left[ \|\Smem_{t,k} - \bars \circ
      \map(\hatS_{t,k-1})\|^2 + \frac{L^2}{\lbatch} \pas_{t,k}^2
      \|\Smem_{t,k} - \hatS_{t,k-1}\|^2 \vert \F_{t,0} \right]\eqsp.
  \end{align*}
  Multiply by $\pas_{t,k+1}$ and apply with $k =
  \xi_t-1$:
 \begin{align*}
& \PE\left[\pas_{t,\xi_t} \|\Smem_{t,\xi_t} - \bars \circ \map(\hatS_{t,\xi_t-1}) \|^2 \vert
      \F_{t,0} \right] \\ & \leq \PE\left[\pas_{t,\xi_t} \|\Smem_{t,\xi_t-1} - \bars \circ
     \map(\hatS_{t,\xi_t-2})\|^2   \vert \F_{t,0} \right] \\
   & + \frac{L^2}{\lbatch} \PE\left[\pas_{t,\xi_t} \pas_{t,\xi_t-1}^2
      \|\Smem_{t,\xi_t-1} - \hatS_{t,\xi_t-2}\|^2 \vert \F_{t,0} \right]\eqsp.
  \end{align*}
This implies
  \begin{align*}
& \bPE_t\left[\pas_{t,\Xi_t} \|\Smem_{t,\Xi_t} - \bars \circ \map(\hatS_{t,\Xi_t-1}) \|^2 \vert
      \F_{t,0} \right] \\ & \leq \bPE_t\left[\pas_{t,\Xi_t} \|\Smem_{t,\Xi_t-1} - \bars \circ
     \map(\hatS_{t,\Xi_t-2})\|^2   \vert \F_{t,0} \right] \\
   & + \frac{L^2}{\lbatch} \bPE_t\left[\pas_{t,\Xi_t} \pas_{t,\Xi_t-1}^2
      \|\Smem_{t,\Xi_t-1} - \hatS_{t,\Xi_t-2}\|^2 \vert \F_{t,0} \right]\eqsp.
  \end{align*}
  By Lemma~\ref{lem:stopping:geom}, we have
\begin{align*}
& \bPE_t\left[\pas_{t,\Xi_t} \|\Smem_{t,\Xi_t-1} - \bars \circ
    \map(\hatS_{t,\Xi_t-2})\|^2 \vert \F_{t,0} \right] \\ & = \rho_t
  \bPE_t\left[\pas_{t,\Xi_t+1} \|\Smem_{t,\Xi_t} - \bars \circ
    \map(\hatS_{t,\Xi_t-1})\|^2 \vert \F_{t,0} \right] \\ & +
  (1-\rho_t) \pas_{t,1} \bPE_t\left[ \|\Smem_{t,0} - \bars \circ
    \map(\hatS_{t,-1})\|^2 \vert \F_{t,0} \right] \eqsp;
\end{align*}
by definition of $\Smem_{t,0}$ and $\F_{t,0}$, the last term is equal
to $(1-\rho_t) \pas_{t,1} \| \mathcal{E}_t \|^2$.
By Lemma~\ref{lem:stopping:geom}, we have
\begin{align*}
& \bPE_t\left[\pas_{t,\Xi_t} \pas_{t,\Xi_t-1}^2 \|\Smem_{t,\Xi_t-1} -
    \hatS_{t,\Xi_t-2}\|^2 \vert \F_{t,0} \right] \\ & = \rho_t
  \bPE_t\left[\pas_{t,\Xi_t+1} \pas_{t,\Xi_t}^2 \|\Smem_{t,\Xi_t} -
    \hatS_{t,\Xi_t-1}\|^2 \vert \F_{t,0} \right] \\ & + (1-\rho_t)
\pas_{t,1} \pas_{t,0}^2 \|\Smem_{t,0} - \hatS_{t,-1}\|^2 \eqsp.
  \end{align*}
 This concludes the proof.
\end{proof}
\begin{lemma}
  \label{lem:majo:gradient}
  For any $h,s,S \in \rset^q$ and any $q \times q$ symmetric matrix $B$, it
  holds
  \begin{align*}
- 2 \pscal{B h }{S} &= -\pscal{B S}{S} - \pscal{B h}{h} + \pscal{B \{h
  - S\}}{h - S} \eqsp.
  \end{align*}
\end{lemma}
\begin{proposition}
  \label{prop:lyap:DL}
  Assume A\ref{hyp:model} to A\ref{hyp:regV}. For any
  $t \in [\kouter]^\star$ and $k \in [\xi_t-1]$,
  \begin{align*}
    &\PE\left[\lyap(\hatS_{t,k+1}) \vert \F_{t,0} \right]
    + \frac{v_{\min}}{2} \pas_{t,k+1} \PE\left[   \|h(\hatS_{t,k})\|^2 \vert \F_{t,0} \right]  \\
    & \leq  \PE\left[\lyap(\hatS_{t,k}) \vert \F_{t,0} \right] \\
    &+   \frac{v_\max}{2} \pas_{t,k+1} \PE\left[ \|\Smem_{t,k+1} - \bars    \circ \map(\hatS_{t,k}) \|^2 \vert \F_{t,0} \right] \\
   & -   \frac{\pas_{t,k+1}}{2} \left( v_{\min} -  \pas_{t,k+1}  L_{\dot  \lyap}\right)  \PE\left[ \|\Smem_{t,k+1} - \hatS_{t,k} \|^2 \vert \F_{t,0}      \right] \eqsp.
    \end{align*}
\end{proposition}
\begin{proof}
Since $\lyap$ is continuously differentiable with $L_{\dot
  \lyap}$-Lipschitz gradient, then for any $s,s' \in \rset^q$,
\[
\lyap(s') - \lyap(s) \leq \pscal{\nabla \lyap(s)}{s'-s} +\frac{L_{\dot  \lyap}}{2} \|s'-s\|^2 \eqsp.
\]
Set $s' = s + pas S$ where $\pas>0$ and $S \in \rset^q$.  Since
$\nabla \lyap(s) = -B(s) h(s)$ and $B(s)$ is symmetric, apply
Lemma~\ref{lem:majo:gradient} with $h \leftarrow h(s)$,
$B \leftarrow B(s)$ and $S = (s'-s)/\pas$; this yields
\begin{align*}
  &\lyap(s+\pas S) - \lyap(s) \leq - \frac{\pas}{2} \pscal{B(s) S}{S} -
  \frac{\pas}{2} \pscal{B(s) h(s)}{h(s)}\\ &+ \frac{\pas}{2}
  \pscal{B(s) \{h(s) - S\}}{h(s) - S} +\frac{L_{\dot  \lyap}}{2} \pas^2
  \|S\|^2 \eqsp.
  \end{align*}
Since $ \|a\|^2 v_{\min} \leq \pscal{B(s)a}{a} \leq v_{\max} \|a\|^2$
for any $a \in \rset^q$, we have
\begin{align*}
  &\lyap(s+\pas S) - \lyap(s) \leq - \frac{\pas v_{\min}}{2} \|S\|^2 -
  \frac{\pas v_{\min}}{2} \|h(s) \|^2 \\ &+ \frac{\pas v_\max}{2} \|
  h(s) - S\|^2 +\frac{L_{\dot  \lyap}}{2} \pas^2 \|S\|^2 \eqsp.
  \end{align*}
  Applying this inequality with $s \leftarrow \hatS_{t,k}$,
  $\pas \leftarrow \pas_{t,k+1}$,
  $S \leftarrow \Smem_{t,k+1} - \hatS_{t,k}$ (which yields
  $s+\pas S = \hatS_{t,k+1}$), and then the conditional expectation
  yield the result.
\end{proof}
\begin{proposition}
 \label{prop:lyap:DL:init}
  Assume A\ref{hyp:model} to A\ref{hyp:regV}. For any $t \in [\kouter]^\star$
\begin{align*}
  &\lyap(\hatS_{t+1,0}) - \lyap(\hatS_{t+1,-1})\\ &\leq -
  \frac{\pas_{t+1,0} v_{\min}}{2} \|h(\hatS_{t+1,-1})\|^2 +
  \frac{v_\max \pas_{t+1,0}}{2} \|\mathcal{E}_{t+1}\|^2 \\ & -
  \frac{\pas_{t+1,0}}{2} \left(v_{\min} - \pas_{t+1,0} L_{\dot
    \lyap} \right) \|\Smem_{t+1,0} - \hatS_{t+1,-1}\|^2 \eqsp.
\end{align*}
\end{proposition}
\begin{proof}
  As in the proof of Proposition~\ref{prop:lyap:DL}, we write for any $s,s' \in \rset^q$,
\[
\lyap(s') - \lyap(s) \leq \pscal{\nabla \lyap(s)}{s'-s} +\frac{L_{\dot  \lyap}}{2} \|s'-s\|^2 \eqsp.
\]
With Lemma~\ref{lem:majo:gradient}, this yields when $s' = s + \pas S$ for $\pas >0$ and $S \in \rset^q$
\begin{align*}
  \lyap(s+\pas S) - \lyap(s) & \leq - \frac{\pas }{2} \left(v_{\min} - \pas
                               L_{\dot  \lyap} \right) \|S\|^2 - \frac{\pas v_{\min}}{2} \|h\|^2
  \\ & + \frac{v_\max \pas}{2} \|h -S\|^2 \eqsp.
\end{align*}
Apply this inequality with $\pas \leftarrow \pas_{t+1,0}$,
$s \leftarrow \hatS_{t+1,-1}$,
$S \leftarrow \Smem_{t+1,0} - \hatS_{t+1,-1}$ and $h \leftarrow h(s)$
(which yields $s+ \pas S = \hatS_{t+1,0}$).
  \end{proof}

\begin{theorem}
  \label{theo:tout}
  Assume A\ref{hyp:model} to A\ref{hyp:regV}. For any
  $t \in [\kouter]^\star$, let $\rho_{t} \in \ooint{0,1}$ and
  $\Xi_t \sim\mathcal{G}^\star(1-{\rho_t})$. Finally, choose
  $\pas_{t,k+1}= \pas_t>0$ for any $k \geq 0$.  For any
  $t \in [\kouter]^\star$, \begin{align*} & \frac{v_{\min}
      \pas_t}{2(1-\rho_t)} \bPE_t\left[ \|h(\hatS_{t,\Xi_t-1})\|^2
      \vert \F_{t,0} \right] \leq \lyap(\hatS_{t,0}) -
    \bPE_t\left[\lyap(\hatS_{t,\Xi_t}) \vert \F_{t,0} \right] \\ &+
    \frac{v_\max}{2(1-\rho_t)} \frac{L^2 }{\lbatch} \pas_{t}
    \pas_{t,0}^2 \|\Smem_{t,0} - \hatS_{t,-1}\|^2 +
    \frac{v_\max}{2(1-\rho_t)} \pas_{t,1} \|\mathcal{E}_t\|^2 \\ & -
    \frac{\pas_{t}}{2(1-\rho_t)} \left( v_{\min} - \pas_{t} L_{\dot
        \lyap} - \frac{v_\max L^2 \rho_t }{(1-\rho_t)
        \lbatch}\pas_t^2\right) \cdots \\ & \qquad \times \bPE_t\left[
      \|\Smem_{t,\Xi_t} - \hatS_{t,\Xi_t-1} \|^2 \vert \F_{t,0}
    \right] \eqsp.
   \end{align*}
\end{theorem}
\begin{proof}
Apply Proposition~\ref{prop:lyap:DL} with $k \leftarrow \xi_t-1$ and
then set $\xi_t \leftarrow \Xi_t$; this yields
 \begin{align*}
    &\bPE_t\left[\lyap(\hatS_{t,\Xi_t}) \vert \F_{t,0} \right] +
   \frac{v_{\min}}{2}  \bPE_t\left[
   \pas_{t,\Xi_t}  \|h(\hatS_{t,\Xi_t-1})\|^2 \vert \F_{t,0} \right] \\ & \leq
   \bPE_t\left[\lyap(\hatS_{t,\Xi_t-1}) \vert \F_{t,0} \right] \\ &+
   \frac{v_\max}{2}  \bPE_t\left[ \pas_{t,\Xi_t}\|\Smem_{t,\Xi_t} - \bars
     \circ \map(\hatS_{t,\Xi_t-1}) \|^2 \vert \F_{t,0} \right] \\ & -
    \bPE_t\left[\frac{\pas_{t,\Xi_t}}{2} \left( v_{\min} - \pas_{t,\Xi_t} L_{\dot
     \lyap}\right) \|\Smem_{t,\Xi_t} - \hatS_{t,\Xi_t-1} \|^2 \vert
     \F_{t,0} \right] \eqsp.
 \end{align*}
 Since $\Xi_t \geq 1$ and $\pas_{t,k} = \pas_t$ for any $k\geq 1$, we have
 \begin{align*}
    &\bPE_t\left[\lyap(\hatS_{t,\Xi_t}) \vert \F_{t,0} \right] +
   \frac{v_{\min}}{2} \pas_{t} \bPE_t\left[ \|h(\hatS_{t,\Xi_t-1})\|^2
     \vert \F_{t,0} \right] \\ & \leq
   \bPE_t\left[\lyap(\hatS_{t,\Xi_t-1}) \vert \F_{t,0} \right] \\ &+
   \frac{v_\max}{2} \pas_{t} \bPE_t\left[ \|\Smem_{t,\Xi_t} - \bars
     \circ \map(\hatS_{t,\Xi_t-1}) \|^2 \vert \F_{t,0} \right] \\ & -
   \frac{\pas_{t}}{2} \left( v_{\min} - \pas_{t} L_{\dot
     \lyap}\right) \bPE_t\left[ \|\Smem_{t,\Xi_t} - \hatS_{t,\Xi_t-1}
     \|^2 \vert \F_{t,0} \right] \eqsp.
 \end{align*}
 By Lemma~\ref{lem:stopping:geom},  it holds
 \begin{align*}
 & \bPE_t\left[\lyap(\hatS_{t,\Xi_t}) \vert \F_{t,0} \right] =
 \bPE_t\left[\lyap(\hatS_{t,\Xi_t-1}) \vert \F_{t,0} \right] \\ &+
 (1-\rho_t) \left( \bPE_t\left[\lyap(\hatS_{t,\Xi_t}) \vert \F_{t,0}
   \right] - \lyap(\hatS_{t,0}) \right)
 \end{align*}
 Furthermore, by Corollary~\ref{coro:squarederror:geom} applied with
 $\pas_{t,\Xi_t} = \pas_{t,\Xi_t+1} = \pas_t$
 \begin{align*}
& \left( 1 - \rho_t\right) \pas_{t} \bPE_t\left[ \|\Smem_{t,\Xi_t} -
     \bars \circ \map(\hatS_{t,\Xi_t-1}) \|^2 \vert \F_{t,0} \right]
   \\ & \leq \frac{L^2 \rho_t}{\lbatch} \pas_{t}^3 \bPE_t\left[
     \|\Smem_{t,\Xi_t} - \hatS_{t,\Xi_t-1}\|^2 \vert \F_{t,0} \right]
   \\ & + \frac{L^2 (1-\rho_t)}{\lbatch} \pas_{t} \pas_{t,0}^2
   \|\Smem_{t,0} - \hatS_{t,-1}\|^2 + (1-\rho_t) \pas_{t,1}
   \|\mathcal{E}_t\|^2 \eqsp,
   \end{align*}
 Therefore,
 \begin{align*}
 & \frac{v_{\min}}{2} \pas_{t} \bPE_t\left[
     \|h(\hatS_{t,\Xi_t-1})\|^2 \vert \F_{t,0} \right] \\ & \leq
   (1-\rho_t) \lyap(\hatS_{t,0}) - (1-\rho_t)
   \bPE_t\left[\lyap(\hatS_{t,\Xi_t}) \vert \F_{t,0} \right] \\ &+
   \frac{v_\max}{2} \frac{L^2 \rho_t}{(1-\rho_t)\lbatch} \pas_{t}^3
   \bPE_t\left[ \|\Smem_{t,\Xi_t} - \hatS_{t,\Xi_t-1}\|^2 \vert \F_{t,0}
     \right] \\ &+ \frac{v_\max}{2} \frac{L^2 }{\lbatch} \pas_{t}
   \pas_{t,0}^2 \|\Smem_{t,0} - \hatS_{t,-1}\|^2 + \frac{v_\max}{2}
   \pas_{t} \|\mathcal{E}_t\|^2 \\ & - \frac{\pas_{t}}{2} \left(
   v_{\min} - \pas_{t} L_{\dot  \lyap}\right) \bPE_t\left[
     \|\Smem_{t,\Xi_t} - \hatS_{t,\Xi_t-1} \|^2 \vert \F_{t,0} \right]
   \eqsp.
 \end{align*}

  This concludes the proof.
\end{proof}
\begin{corollary}[of Theorem~\ref{theo:tout}]
  For any $t \in [\kouter]^\star$,
 \begin{align*}
 & \left( \frac{\pas_t\rho_t}{(1-\rho_t)} +\pas_{t+1,0}\right) \frac{
     v_{\min}}{2} \bPE_t\left[\|h(\hatS_{t,\Xi_t})\|^2 \vert \F_{t,0}
     \right] \\ & \leq \lyap(\hatS_{t,0}) - \bPE_t\left[
     \lyap(\hatS_{t+1,0}) \vert \F_{t,0} \right] \\ & -
   \frac{\pas_{t+1,0}}{2} \left(v_{\min} - \pas_{t+1,0} L_{\dot
     \lyap} \right) \PE\left[\| \Smem_{t+1,0} -
     \hatS_{t+1,-1}\|^2\vert \F_{t,0} \right] \\ &+
   \frac{v_\max}{2(1-\rho_t)} \frac{L^2 }{\lbatch} \pas_{t}
   \pas_{t,0}^2 \|\Smem_{t,0} - \hatS_{t,-1}\|^2 \\ & +
   \frac{v_\max}{2(1-\rho_t)} \pas_{t} \|\mathcal{E}_t\|^2 +
   \frac{v_{\max} \pas_{t+1,0}}{2} \PE\left[\|
     \mathcal{E}_{t+1}\|^2 \vert \F_{t,0} \right] \eqsp.
   \end{align*}
\end{corollary}
\begin{proof}
 Let $t \in [\kouter]^\star$. By Proposition~\ref{prop:lyap:DL:init},
 since $\hatS_{t,\xi_t} = \hatS_{t+1,-1}$ we have
\begin{align*}
  & - \PE\left[\lyap(\hatS_{t,\xi_t}) \vert \F_{t,0} \right] \leq -
  \PE\left[ \lyap(\hatS_{t+1,0}) \vert \F_{t,0} \right] \\ & -
  \frac{\pas_{t+1,0} v_{\min}}{2} \PE\left[\|h(\hatS_{t,\xi_t})\|^2
    \vert \F_{t,0} \right]  \\
  & + \frac{v_{\max} \pas_{t+1,0}}{2}
  \PE\left[\| \mathcal{E}_{t+1}\|^2 \vert \F_{t,0} \right] \\ & -
  \frac{\pas_{t+1,0}}{2} \left(v_{\min} - \pas_{t+1,0} L_{\dot
    \lyap} \right) \PE\left[\|\Smem_{t+1,0} - \hatS_{t+1,-1}\|^2\vert
    \F_{t,0} \right] \eqsp.
\end{align*} 
The previous inequality remains true when
$\PE\left[\lyap(\hatS_{t,\xi_t}) \vert \F_{t,0} \right]$ is replaced
with $\bPE_t\left[\lyap(\hatS_{t,\Xi_t}) \vert \F_{t,0} \right]$; and
$\PE\left[\|h(\hatS_{t,\xi_t})\|^2 \vert \F_{t,0} \right] =
\bPE_t\left[\|h(\hatS_{t,\Xi_t})\|^2 \vert \F_{t,0} \right]$.
The proof follows from Theorem~\ref{theo:tout}, and (see
Lemma~\ref{lem:stopping:geom})
\[
\bPE_t\left[ \|h(\hatS_{t,\Xi_t-1})\|^2 \vert \F_{t,0} \right] \geq
\rho_t \bPE_t\left[ \|h(\hatS_{t,\Xi_t})\|^2 \vert \F_{t,0} \right] \eqsp.
\]
\end{proof}

\end{document}